\documentclass[11pt]{article}
\usepackage[utf8]{inputenc}
\usepackage[T1]{fontenc}
\usepackage{amsmath,amssymb,amsthm}
\usepackage[margin=1.1in]{geometry}
\usepackage{booktabs}
\usepackage{graphicx}
\usepackage{natbib}
\usepackage{hyperref}
\hypersetup{colorlinks=true,linkcolor=blue,citecolor=blue,urlcolor=blue}

\newtheorem{theorem}{Theorem}
\newtheorem{proposition}[theorem]{Proposition}

\newtheorem{conjecture}[theorem]{Conjecture}
\theoremstyle{definition}
\newtheorem{definition}[theorem]{Definition}
\newtheorem{remark}[theorem]{Remark}

\newcommand{\Inc}{\mathrm{Inc}}
\newcommand{\conv}{\mathrm{conv}}
\newcommand{\R}{\mathbb{R}}
\newcommand{\E}{\mathbb{E}}

\title{Understanding as No-Arbitrage:\\ Bounded Dutch Books as a Definition and Training Objective\\ for Language Models}
\author{Daniel Dragonevskiy\\
  \normalsize \texttt{ispromashka@gmail.com}}
\date{September 2026}

\begin{document}
\maketitle

\begin{abstract}
Does a language model merely predict tokens, or does it understand what it says? The question stays undecidable until understanding is given some operational meaning. We propose one, borrowing from inferential-role semantics and the oldest device for keeping graded commitments honest: the Dutch book. A model understands a vocabulary, relative to an inference system and a degree $(k,\varepsilon)$, if no trader from a bounded class $\mathcal{T}_k$ can extract guaranteed profit above $\varepsilon$ by betting against its elicited credences on logically related claims. That makes understanding measurable, down to a single word or relation.

Four results follow. A ladder theorem ties bounded trader classes to local relaxations of the marginal polytope: full coherence is \textsf{NP}-hard to check, so any bounded system can only understand by degrees, and the theorem predicts how that degree should generalize. A second theorem shows that the exact optimum of negative log-likelihood, on a corpus of mutually inconsistent but stylistically distinguishable authors, is incoherent across elicitation contexts with a guaranteed profit computable in closed form: the defect sits in the training objective, not the architecture, and scale does not remove it. A $\delta$-generalization of Adams' theorem shows coherent credences degrade additively along derivation chains, and that overconfidence about a derived claim becomes an arbitrage opportunity in its own right, unlike the silent error compounding of autoregressive sampling. Finally we specify a training scheme, \textsc{Arbitr}, in which a trader adversary supplies the missing cross-context term.

Five pre-registered experiments test this. E-A0 measures exploitability across models from 1.5B to 7B parameters: within-format credences are already close to coherent, but cross-format and negation books extract guaranteed profit up to $0.15$ per unit stake, worst at 7B on facts it is unsure of. E-A1 trains against the exact adversary and cuts negation exploitability two orders of magnitude, transferring to an untrained, higher-level pattern in every seed, but half the runs collapse into an uninformative equilibrium, exactly as predicted for an objective missing its calibration term. E-A1b adds that term; every registered criterion then passes, at no cost to task accuracy. E-A2 replaces the hand-specified adversary with a learned one that discovers exploitable books on its own, and shows partial transfer to 7B. E-A3 tightens the attribution. An ablation without the trader shows that the cross-format gain comes from the adversarial term rather than the calibration anchor, while on negation the anchor alone accounts for a factor of about $2.4$; the mechanism replicates on a second model family (Phi-3.5) and carries over to a negation template never seen in training, with a measurable residue of wording-specificity. It also exposes the approach's sharpest limit: at 7B, near-zero measured incoherence coincides with near-maximal confidence on most items, so part of the apparent coherence at scale is confident degeneration rather than calibration. We are careful about what this does not show: factual hallucination is untouched, and coherence is necessary for knowledge, not sufficient.
\end{abstract}

\section{Introduction}

Are large language models ``stochastic parrots'' that predict tokens without understanding, or does next-token prediction suffice for understanding on its own? The debate has run for years without a criterion either side could lose by \citep{bender2020climbing,mitchell2023debate,sogaard2025semantics}. Probing work has weakened the strong ``blind correlation'' reading: next-token predictors pick up linearly decodable, causally effective world models \citep{li2023othello,nanda2023othello,gurnee2024space}. Yet the same models routinely violate elementary probability identities across logically related questions \citep{zhu2024incoherent,andrews2026dutchbooks}, and their multi-step reasoning degrades in ways their stated confidence never reflects \citep{dziri2023faith,bachmann2024pitfalls}.

We think the more useful question is not whether LLMs understand, but what measurable property would count as understanding, and what training objective would produce it. Inferential-role semantics ties the meaning of an expression to the inferences it licenses \citep{brandom1994}; we operationalize that idea with the oldest formal tool for keeping graded commitments honest, the Dutch book \citep{definetti1937}. The resulting definition is behavioral, graded, word-level, and---the part that matters for training---adversarially trainable.

\paragraph{Definition (informal).} Fix an inference system $I$ and a family of elicitation contexts. A model understands a vocabulary to degree $(k,\varepsilon)$ if no trader in a bounded class $\mathcal{T}_k$ (portfolios supported on at most $k$ propositions instantiating rules of $I$, including cross-context ``law of one price'' books) achieves guaranteed profit exceeding $\varepsilon$ against the model's elicited credences.

\paragraph{Contributions.} Exploitability-as-distance is classical \citep{schervish2002measures,debona2015measuring}, and we do not claim it. What we do claim:
\begin{enumerate}
  \item \textbf{A graded ladder} (Thm.~\ref{thm:ladder}): $\mathcal{T}_k$-unexploitability turns out to equal $k$-local consistency, membership in a local relaxation of the marginal polytope. The polyhedral combinatorics here is standard (the local polytope of \citealp{wainwright2008}; Sherali--Adams relaxations \citealp{sherali1990}; PSAT complexity \citealp{georgakopoulos1988,pitowsky1991}), and we are not claiming any of it as new. What is new, we think, is the dictionary between bounded traders and relaxation levels, and its consequence: since full coherence is \textsf{NP}-complete, a bounded agent can only understand by degrees, which turns ``degree of understanding'' into a complexity-theoretic quantity with a generalization prediction attached. Sum-of-squares pseudodistributions as ``computationally bounded Bayesian beliefs'' \citep{barak2014sos} and bounded-trader logical induction \citep{garrabrant2016} anticipate the idea conceptually.
  \item \textbf{An objective-defect theorem} (Thm.~\ref{thm:nll}): on a mixture corpus of internally consistent, mutually inconsistent authors whose style correlates with identity, the exact NLL optimum turns out to be incoherent across contexts, with guaranteed arbitrage profit equal to the style--author correlation. Per context the optimum is coherent---a mixture of coherent credences is coherent for unconditional events, cf.\ opinion-pooling results \citep{madansky1964,genest1986}---so the defect is precisely located: NLL never scores one proposition across two contexts. No amount of data or scale removes this; only a cross-context term added to the objective can.
  \item \textbf{Priced additive degradation} (Thm.~\ref{thm:adams}), a $\delta$-generalization of Adams' theorem \citep{adams1975,gilio2002}: $\delta$-coherent credences degrade at most additively along derivation chains, and overstating terminal confidence relative to premise prices is itself an arbitrage opportunity. That degradation is visible in the price is the useful part---autoregressive sampling compounds errors silently \citep{dziri2023faith,lecun2023}.
  \item \textbf{A training scheme} (\textsc{Arbitr}, \S\ref{sec:arbitr}): a small trader network learns to construct profitable books over batches whose admissible worlds are known by construction, and the model is penalized by the realized profit. Coherentizing at any level cannot hurt proper-score accuracy (a small extension of \citealp{predd2009}, Prop.~\ref{prop:dominance}), so the two loss terms do not fight each other at the optimum. Bid--ask spreads (lower/upper previsions, \citealp{walley1991}) give ``I have no basis to commit'' an architectural home; calibration bettors \citep{foster1998} punish unfounded precision and keep the spread from collapsing under gradient pressure. Per-word exploitability $\Inc_w$ gives individual lexical items their own understanding certificates.
\end{enumerate}
We also ran a pre-registered measurement, E-A0, of $\mathcal{T}_2$/$\mathcal{T}_3$-exploitability across model scales (\S\ref{sec:ea0})---the cheapest way to find out whether the whole program is moot, i.e.\ whether deployed models were already close enough to coherent that none of this matters.

\section{Related work}\label{sec:related}

\paragraph{Incoherence measures.} Measuring incoherence via normalized guaranteed Dutch-book loss goes back to \citet{schervish2002measures,schervish2003measures}. \citet{debona2015measuring} show that distance-based and Dutch-book-based inconsistency measures coincide and are LP-computable; \citet{potyka2014} studies $p$-norm minimal-violation measures; \citet{staffel2015,staffel2019} work out the philosophy. Our Prop.~\ref{prop:linf} restates this lineage compactly. \citet{andrews2026dutchbooks} compute LP Dutch books against LLMs, as a measurement only. \citet{andrews2026revealed} proposes LP-computable penalties from representation theorems (de Finetti, Afriat, Echenique--Saito) as label-free evaluation metrics and candidate regularizers, with the bettor appearing only as an LP variable---no bounded trader classes, no learnable adversary, no training experiments. That is the gap this paper's method and experiments fill.

\paragraph{Probability logic.} Additive degradation under $p$-validity is Adams' theorem \citep{adams1975}; precise propagation of bounds in System~P is due to \citet{gilio2002}; probabilistic entailment as LP is \citet{nilsson1986,hailperin1996}; PSAT complexity is \citet{georgakopoulos1988}; correlation polytopes and their local relaxations are \citet{pitowsky1991}. Approximate coherence and its graded benefits: \citet{debona2017approximately}.

\paragraph{Bounded rationality via traders and pseudodistributions.} Logical induction defines rational credence by unexploitability against polynomial-time traders \citep{garrabrant2016}; sum-of-squares pseudodistributions formalize beliefs of computationally bounded Bayesians \citep{barak2014sos}. We borrow both intuitions and put them to work at finite, learnable scale.

\paragraph{LLM (in)coherence and consistency training.} Systematic probability-identity violations in LLMs: \citet{zhu2024incoherent}; non-Bayesian updating: \citet{llmnotbayesian2026}. Consistency has been enforced at decode time \citep{kassner2021belief,mitchell2022concord} or by losses tied to fixed fact/rule sets \citep{calanzone2024}; paraphrase-consistency losses: \citet{elazar2021}. RL-style ``self-consistency'' rewards agreement among samples, not probabilistic coherence over related propositions \citep{selfctrl2026}. Semantic entropy \citep{farquhar2024} is the special case of our law-of-one-price books restricted to paraphrase families. \citet{betz2023} train coherent credences on synthetic corpora without an adversary. Negation-driven, objective-level failures: \citet{asher2024negation}. To put the kinship plainly: the training component here is consistency fine-tuning in the lineage of \citet{elazar2021} and \citet{calanzone2024}, generalized to a market semantics. What is added is an adversary that searches for violations instead of enumerating them, settlement by admissible worlds instead of a fixed fact list, a graded trader-class ladder with its own generalization prediction, and a no-arbitrage reading that ties the loss back to a definition.

\paragraph{World models.} The world-model program \citep{lecun2022path} goes after the same slogan, ``prediction is not understanding,'' architecturally: replace token prediction with latent-state prediction and planning. Our results sit next to that, not against it---coherence of elicited credences is something any architecture serving as an epistemic interface needs, and Thm.~\ref{thm:nll} is about the objective given an architecture, not about the architecture itself. Nothing here bears on whether latent world models are needed for grounding or planning. But a JEPA-style system that answers questions still owes its users credences that don't contradict themselves across phrasings.

\paragraph{Credal and adversarial learning.} Credal deep learning trains classifiers to output probability intervals \citep{caprio2023,creinns2024,credalensembles2024}; Credal LLMs \citep{credalllm2026} derive intervals post hoc from adapter ensembles. None provide avoiding-sure-loss guarantees or an anti-collapse mechanism; none train an LM end to end. Adversarial training with logical opponents exists in non-probabilistic form (Ehrenfeucht--Fra\"iss\'e games over graphs with a depth-$k$ curriculum, \citealp{logan2025}) and as LLM discriminators of reasoning soundness without stakes \citep{gar2025}.

\paragraph{Hallucination lower bounds.} Calibrated models forced to answer must hallucinate on missing mass \citep{kalai2024}; these bounds don't care about representation or objective. Nothing below contradicts them---coherence is a different failure class (\S\ref{sec:limits}).

\section{Setup}\label{sec:setup}

Let $\Phi$ be a finite set of natural-language propositions carrying a known logical structure (atoms, connectives, certified paraphrase/entailment relations from an inference system $I$). A \emph{world} is an assignment $w \in \{0,1\}^\Phi$ consistent with $I$; $W$ denotes the set of worlds and $C := \conv\{w : w \in W\}$ the coherent polytope (de Finetti): $q \in C$ iff $q$ extends to a probability measure over worlds.

A model $M$ induces \emph{credences} $q_M(\varphi; c) \in [0,1]$ through an \emph{elicitation family}: contexts $c \in E(\varphi)$ (question formats, languages, paraphrastic framings, irrelevant-context perturbations) mapped to the model's normalized probability of an affirmative answer. Coherence of the \emph{assembled} credence function requires both (a) membership in $C$ and (b) \emph{law of one price} (LOP): $q(\varphi;c) = q(\varphi;c')$ for $c,c' \in E(\varphi)$. LOP books are the cross-context instruments; semantic entropy \citep{farquhar2024} measures exactly LOP dispersion over paraphrase families.

A \emph{portfolio} is $x \in \R^\Phi$, $\|x\|_1 \le 1$: buy $x_\varphi$ shares of $\varphi$ at price $q(\varphi)$; a share pays 1 iff $\varphi$ holds. Guaranteed profit: $g(x) = \min_{w \in W} \sum_\varphi x_\varphi (w(\varphi) - q(\varphi))$; exploitability $\Inc(q) := \max_{\|x\|_1 \le 1} g(x)$.

\begin{definition}[Bounded trader classes]
$\mathcal{T}_k$ is the set of portfolios whose support lies in some $S \in \mathcal{S}_k$, where $\mathcal{S}_k$ collects supports of size $\le k$ instantiating rules of $I$ with $\le k$ premises (negation and LOP pairs at $k{=}2$; conjunction, modus ponens, transitivity triples at $k{=}3$; longer chains and quantifier instances above). $\Inc_k(q)$ is the corresponding restricted maximum.
\end{definition}

\begin{definition}[Graded understanding]\label{def:understanding}
$M$ \emph{understands} $\Phi$ to degree $(k,\varepsilon)$, relative to $I$ and the elicitation family, iff $\Inc_k(q_M) \le \varepsilon$. For a word $v$, the \emph{understanding certificate} $\Inc_{k,v}$ restricts the maximum to books pivoting on $v$ (supports whose propositions differ in the inferential role of $v$).
\end{definition}

\section{Results}\label{sec:results}

\begin{proposition}[Exploitability is distance; known]\label{prop:linf}
$\Inc(q) = \min_{\mu \in C} \|q - \mu\|_\infty$. In particular $\Inc(q)=0$ iff $q \in C$.
\end{proposition}
\begin{proof}
For fixed $x$, $\min_{w \in W}\langle x, w-q\rangle = \min_{\mu \in C}\langle x, \mu - q\rangle$ (a linear function attains its minimum on a polytope at a vertex). Both feasible sets are convex and compact and the objective is bilinear, so Sion's minimax theorem permits the exchange:
$\Inc(q) = \min_{\mu\in C}\max_{\|x\|_1\le 1}\langle x,\mu-q\rangle = \min_{\mu\in C}\|\mu-q\|_\infty$.
\end{proof}
\begin{remark}
This is a short route to a known equivalence between Dutch-book and distance-based incoherence measures \citep{debona2015measuring,schervish2002measures,potyka2014}. We use it here only as a building block.
\end{remark}

\begin{theorem}[Ladder of understanding]\label{thm:ladder}
(a) $\Inc_k(q) = \max_{S \in \mathcal{S}_k} \; \mathrm{dist}_\infty\!\big(q|_S,\ \conv(W|_S)\big)$: bounded exploitability equals the largest violation of $k$-local coherence constraints, i.e.\ $\mathcal{T}_k$-unexploitability is membership in a local (Sherali--Adams-style) relaxation of the marginal polytope. (b) $\Inc_2 \le \Inc_3 \le \dots \le \Inc$, with equality at $k = |\Phi|$. (c) Verifying $\Inc_k(q) \le \varepsilon$ for fixed $k$ is polynomial; verifying $\Inc(q)=0$ is \textsf{NP}-complete \citep{georgakopoulos1988,pitowsky1991}.
\end{theorem}
\begin{proof}
(a) Apply Prop.~\ref{prop:linf} to each admissible support $S$, noting $\conv(W)|_S = \conv(W|_S)$ under projection, and that a portfolio confined to $S$ interacts with worlds only through $W|_S$. (b) Nesting of portfolio classes. (c) For fixed $k$, $|\mathcal{S}_k|$ is polynomial in $|\Phi|$ and each local LP has bounded size; global coherence is PSAT.
\end{proof}
\begin{remark}
Since full understanding is NP-hard, any bounded agent, human or machine, can only be coherent up to some level of the hierarchy. ``Degrees of understanding'' is a rung, not a figure of speech. This gives a testable prediction (\S\ref{sec:ea0}): a model trained against $\mathcal{T}_k$ should generalize coherence to unseen instances of patterns at level $\le k$, not at level $k{+}1$.
\end{remark}

\begin{theorem}[The exact NLL optimum is cross-context incoherent]\label{thm:nll}
Let the corpus be generated by two internally consistent authors: author 1 asserts $\varphi$, author 2 asserts $\neg\varphi$, with equal weights, and let style correlate with identity: contexts $c_1, c_2$ with $P(c_1 \mid \mathrm{author}\,1) = P(c_2 \mid \mathrm{author}\,2) = (1+\rho)/2$, $\rho > 0$. The exact NLL optimum $M^\ast$ (the true conditional distribution) elicits $q(\varphi; c_1) = (1+\rho)/2$ and $q(\varphi; c_2) = (1-\rho)/2$, and the LOP book (sell $\varphi$ at $c_1$, buy at $c_2$) collects guaranteed profit $\rho$ regardless of the truth value of $\varphi$: $\Inc_{\mathrm{LOP}}(M^\ast) \ge \rho$.
\end{theorem}
\begin{proof}
$q(\varphi; c_i) = P(\mathrm{author}\,1 \mid c_i)$ by Bayes' rule; the two share positions in $\varphi$ cancel in every world, leaving the price difference $\rho$.
\end{proof}
\begin{remark}[Locating the defect]\label{rem:location}
The arithmetic here is elementary, Bayes' rule and cancellation, and per context $M^\ast$ is coherent: a mixture of coherent credences is coherent for unconditional events (the coherent set is convex; pooling only breaks under conditionalization and across contexts, \citealp{madansky1964,genest1986}). The real content is that no NLL optimum on such a corpus can be a faithful corpus simulator and an elicitation-invariant credence function at the same time.

One objection deserves an answer. $q(\varphi;c)$ is a conditional probability, and conditioning on informative context is not irrational---in the corpus, style really is evidence about the author. That's true, and it's the point. The law-of-one-price constraint only applies across elicitations the deployer has declared meaning-preserving. At deployment, the language or format of a question is the asker's free choice, not a sample from the corpus's author process, and it carries no evidence about $\varphi$. The simulator has no way to know this: the NLL objective never sees the elicitation family. So the corpus-rational conditional and the deployment-rational credence pull apart, and the gap is real in practice---a cross-format book collects its profit from the deployed system whatever the rationale behind its prices was (cf.\ the simulator-vs-agent distinction of \citealp{andreas2022}). Read ``defect,'' then, as a mismatch between what the objective rewards (simulation) and what deployment needs (assistantship), not as an accusation of Bayesian error. Scale and data can't close that gap; only a cross-context term in the objective can. Within-context violations seen in practice \citep{zhu2024incoherent} are a separate, capacity/representation story, with negation a proven mechanism \citep{asher2024negation}---complementing Thm.~\ref{thm:nll} rather than competing with it.
\end{remark}

\begin{theorem}[Priced additive degradation; $\delta$-Adams]\label{thm:adams}
Say $q$ is $\delta$-coherent w.r.t.\ $I$ if for every rule instance $\Gamma \vdash \psi$, $q(\psi) \ge 1 - \sum_{\gamma \in \Gamma} (1 - q(\gamma)) - \delta$. Let $\psi_0, \dots, \psi_n$ be a derivation chain in which step $i$ uses $\psi_i$ and at most $m$ side premises, each priced $\ge 1 - \epsilon$. Then
\[ 1 - q(\psi_n) \;\le\; \big(1 - q(\psi_0)\big) + n\,(m\epsilon + \delta). \]
Moreover, any assembled credence assigning $\psi_n$ a price exceeding the bound implied by the enforced constraints admits a book with guaranteed profit equal to the excess (by Thm.~\ref{thm:ladder}(a) applied to the chain's supports).
\end{theorem}
\begin{proof}
Induction over steps using the $\delta$-inequality; the second claim is the definition of local exploitability.
\end{proof}
\begin{remark}
At $\delta = 0$ this is the classical uncertainty-accumulation result \citep{suppes1966,adams1975,adamslevine1975}, and approximate propagation of uncertainty through inference rules is itself old news---Gilio's coherence-based System~P bounds \citep{gilio2000,gilio2002}, Hailperin's optimal LP bounds \citep{hailperin1996}, the graded benefits of approximate coherence \citep{debona2017approximately}. We're not claiming novelty for the additive skeleton or the $\varepsilon$--$\delta$ machinery. What the trader vocabulary adds is the pricing corollary: in autoregressive sampling a step error enters the conditioning context silently, and terminal confidence carries no trace of it \citep{dziri2023faith,bachmann2024pitfalls}. Under enforced coherence constraints, understated degradation becomes an arbitrage opportunity instead---an audit-LP catches it at inference, a trader adversary penalizes it during training. Combine this with a selective rule (assert only if $q \ge 1-\alpha$) and calibration on the assertion set, and asserted $n$-step conclusions come with a price-backed error bound.
\end{remark}

\begin{proposition}[Level-wise dominance; extension of \citealp{predd2009}]\label{prop:dominance}
Let $K \supseteq C$ be any closed convex set of a ladder level (a local relaxation), and let $q \notin K$. The Euclidean projection $\pi_K(q)$ satisfies, for every world $w \in W \subseteq K$,
$\|\pi_K(q) - w\|^2 \le \|q - w\|^2 - \|q - \pi_K(q)\|^2$:
coherentization at any level strictly improves Brier score in every possible world. Hence an arbitrage penalty does not trade off against proper-score accuracy at the optimum.
\end{proposition}
\begin{proof}
The obtuse-angle property of projections onto convex sets containing $w$.
\end{proof}

\begin{proposition}[Spreads; restatement of \citealp{walley1991}]\label{prop:walley}
Interval credences $[l(\varphi), u(\varphi)]$ avoid sure loss iff there exists $p \in C$ with $l \le p \le u$ pointwise. The \emph{natural extension}---the LP bounds $[\min, \max]\, p(\psi)$ over $\{p \in C : l \le p \le u \text{ on quoted propositions}\}$---is the canonical inference from partial commitments.
\end{proposition}
\begin{remark}
The spread gives ``I have no basis to commit'' an explicit home, with exact semantics: a wide interval with a nonempty core is honest ignorance; an empty core is incoherence, whatever the width. Inference becomes an amortized natural extension---the trained model approximates the LP bounds in a forward pass, and the exact LP audits it. Credal deep learning already trains interval outputs for classifiers \citep{caprio2023,creinns2024,credalensembles2024}; what's new here is doing it end-to-end for an LM, the avoiding-sure-loss criterion, and the anti-collapse game below.
\end{remark}

\begin{conjecture}[Equilibrium spread]\label{conj:spread}
Under a proper-scoring/coverage pressure (narrowing) opposed by logical traders and calibration bettors \citep{foster1998} settling on a verifiable anchor set (widening where unfounded), equilibrium spread width on $\varphi$ converges to the epistemic uncertainty consistent with the anchor and $\varphi$'s logical relations to it. In the exchangeable toy case (Beta posterior over a settleable event family), the equilibrium width equals the credible-interval width.
\end{conjecture}
\noindent This conjecture is the design's real bet: it's what should keep learned ignorance from collapsing under gradient pressure. Zeroing the spread without evidence creates exploitable calibration books---the loss goes up, not down.

\section{The \textsc{Arbitr} objective}\label{sec:arbitr}

\emph{Batch construction.} A generator instantiates rule patterns of $I$ over templated content---fictional entities to isolate structure from knowledge, real-world content to probe what the deployed model actually believes---producing supports whose world sets are known by construction. Settlement of guaranteed profit is exact and cheap at training time this way; the \textsf{NP}-hardness of Thm.~\ref{thm:ladder}(c) only bites in free text, where the trained trader has to generalize and an exact audit-LP over the answer's neighborhood does the certifying.

\emph{Players.} The model quotes $q(\varphi; c)$ (or $[l,u]$) across the elicitation family. A small trader network sees the batch, its structure, and the quotes, outputs a portfolio in $\mathcal{T}_k$, and is trained to maximize realized guaranteed profit; the model minimizes $\mathcal{L}_{\mathrm{task}} + \lambda \, \E[\text{profit}]$, plus proper scoring on the anchor set. Training alternates GAN-style, with a curriculum climbing the ladder in $k$. Prop.~\ref{prop:dominance} predicts no task regression at the optimum. The experiments below test that directly, and one of them (E-A1) shows it fails once the anchoring premise is dropped.

\emph{Certificates.} At inference: (i) per-word $\Inc_{k,v}$ maps (``the model does not understand \emph{larger}: a transitivity book yields 0.31''); (ii) an audit-LP over the local neighborhood of an answer, run by a frozen tuple (weights $\times$ elicitation family $\times$ audit-LP $\times$ anchor $\times$ decision rule)---never by the model's own trader; (iii) abstention when the spread exceeds threshold or a book is found, with Thm.~\ref{thm:adams} supplying the selective guarantee.

\section{E-A0: measuring the target}\label{sec:ea0}

Before training anything, the program must establish that deployed credences are exploitable enough to matter---the pre-registered kill threshold is median $\Inc < 0.05$ across all models and formats, in which case the program closes.

\emph{Design} (pre-registered).\ Five support patterns---negation, certified paraphrase (LOP), conjunction, transitivity (core), and modus ponens (exploratory: natural-language conditionals are pragmatically confounded w.r.t.\ material implication)---over two content classes (fictional entities; real-world comparatives and capitals), 600 supports, three elicitation formats (F1: Russian yes/no; F2: English true/false; F3: Russian numbered choice), each with and without an irrelevant-context distractor \citep{andrews2026dutchbooks}. Credences are read from the first-token distribution over answer-variant tokens; supports with answer-mass coverage $< 0.5$ are excluded (fraction reported). Exploitability per support is the exact LP of Prop.~\ref{prop:linf}; cross-format LOP books price one proposition across formats. Models: Qwen2.5-\{1.5B, 3B, 7B\}-Instruct and SmolLM3-3B (non-Qwen arm); 8$\times$V100 node; $9{,}360$ elicitations per model. The pre-registered primary metric was the median support-level $\Inc$ over pooled core patterns (no distractor), with kill threshold $0.05$.

\paragraph{Pre-registered verdict.} \textbf{The kill criterion fired}: primary medians came in at $0.024$ [CI95 $0.019, 0.032$] for 1.5B, $0.000$ for 3B, $0.0001$ for 7B, and $0.019$ for SmolLM3 ($27\%$ excluded by coverage). By the registered criterion, within-format, template-level incoherence is negligible for instruction-tuned models at 1.5B--7B. We treat this verdict as binding for the metric as registered; everything from here is post-hoc diagnosis.

\begin{table}[h]\centering\small
\resizebox{\textwidth}{!}{%
\begin{tabular}{lcccc}
\toprule
 & Qwen2.5-1.5B & Qwen2.5-3B & Qwen2.5-7B & SmolLM3-3B \\
\midrule
Primary median $\Inc$ (pre-reg.) & 0.024 & 0.000 & 0.0001 & 0.019 \\
\midrule
Negation channel, F1 / F2 (median) & 0.472 / 0.165 & 0.177 / 0.001 & 0.038 / 0.105 & 0.052 / 0.117 \\
Conditionals (MP, exploratory) & 0.345 & 0.500 & 0.495 & 0.439 \\
\midrule
Cross-format book F1$\leftrightarrow$F2: mean profit & 0.151 & 0.021 & 0.064 & 0.070 \\
\quad share of propositions with profit $>0.25$ & 25.0\% & 4.0\% & 13.3\% & 7.5\% \\
\bottomrule
\end{tabular}}
\caption{E-A0 (v2 run). Top: the pre-registered primary metric---the kill threshold of $0.05$ fired for every model. Bottom: a post-hoc channel breakdown, restricted to the two readout-valid formats, no distractor. Cross-format profits are per unit stake on law-of-one-price books.}
\label{tab:ea0}
\end{table}

\paragraph{Post-hoc diagnosis.} Three things survive scrutiny here. The aggregator, not the target, was the problem: pooled medians are dragged down by patterns---paraphrase, conjunction, transitivity---that a confidently degenerate responder satisfies for free, so exploitability concentrates in specific channels and tails rather than showing up in an average. The negation channel ($|q(\varphi)+q(\neg\varphi)-1|/2$) is large at 1.5B and still present at 7B (Table~\ref{tab:ea0}); the conditional channel is large everywhere, but confounded by conditional pragmatics as expected going in.

The dominant phenomenon turns out to be cross-elicitation incoherence, not within-format incoherence. Between the two individually valid formats, law-of-one-price books collect mean guaranteed profit of $0.15$ at 1.5B and $0.064$ at 7B, with $13$--$25\%$ of propositions yielding profit above $0.25$; the effect doesn't shrink monotonically with scale (7B is worse than 3B). A model can be nearly coherent inside one format and still be incoherent across two phrasings of the same question. That's exactly the failure Thm.~\ref{thm:nll} predicts for NLL-trained models---coherent per context, but assembled in a context-sensitive way---though matching a post-hoc pattern to a theorem is not the same as testing it.

Readout validity turned out to matter as much as anything we set out to measure. Three silent elicitation failures showed up during the campaign, and the coverage filter caught none of them on its own: a Cyrillic letter-option that collided with the first token of the Russian word for ``true'' (which inverted answers at high coverage, so the first run was discarded); a reasoning-mode chat template that consumed the first answer token; and a numbered-choice format whose answers barely correlated with the other two ($r \approx 0.02$--$0.4$) and flipped under an irrelevant preamble. What caught all three was the cross-format books pricing them---which is itself a decent argument for auditing by book rather than by filter.

\paragraph{Relation-level certificates.} The pre-registration promised word/relation-level exploitability maps as a byproduct, and Table~\ref{tab:certs} delivers them: negation and paraphrase books pivoting on each relation, valid formats only, no distractor. They pin down where the headline anomaly at 7B actually lives---the model is close to perfectly coherent on capitals and age order, but its median book profit on mountain-height and river-length comparatives is $0.43$--$0.48$. Its cross-format instability sits precisely on the relations it doesn't know for sure. In the vocabulary of Def.~\ref{def:understanding}: it understands \emph{capital-of} and fails to understand \emph{higher-than}, and the certificate says by how much.

\begin{table}[h]\centering\small
\begin{tabular}{lccccc}
\toprule
median $\Inc$ by relation & \emph{capital-of} & \emph{higher-than} & \emph{longer-than} & \emph{older-than} & \emph{ball-in-box} \\
\midrule
Qwen2.5-1.5B & 0.027 & 0.249 & 0.081 & 0.075 & 0.091 \\
Qwen2.5-3B & 0.000 & 0.000 & 0.000 & 0.000 & 0.001 \\
Qwen2.5-7B & 0.000 & \textbf{0.428} & \textbf{0.482} & 0.000 & 0.007 \\
SmolLM3-3B & 0.002 & 0.115 & 0.038 & 0.062 & 0.045 \\
\bottomrule
\end{tabular}
\caption{Relation-level understanding certificates from E-A0: median violation on negation/paraphrase books pivoting on each relation (F1/F2, no distractor, $n \approx 40$--$100$ books per cell).}
\label{tab:certs}
\end{table}

\paragraph{Consequence for the program.} E-A0 kills the naive idea that within-format logical books alone give a rich training signal at 3B and up. It points the trader class toward cross-elicitation books and the negation channel instead, and it meant a fresh pre-registration was needed before any training experiment. The exploratory/confirmatory line stayed clean throughout: nothing from the post-hoc reading above went directly into a confirmatory claim. Each later experiment (E-A1, E-A1b, E-A2) registered its own endpoints and thresholds before running, and we report their verdicts by those registrations, failures included.

\section{E-A1: training against the exact bounded adversary}\label{sec:ea1}

Following on from E-A0, we ran a pre-registered, minimal training experiment. Qwen2.5-1.5B-Instruct, LoRA ($r{=}16$), fp32, 1000 steps, four seeds, two FLOP-matched arms fed identical batches. Arm A (control) trains only a task loss---cross-entropy toward the true yes/no token mass on real-world facts of known truth. Arm B adds $\lambda{=}1$ times the profit of the exact best-response $\mathcal{T}_2$ trader: negation books $|p(\varphi){+}p(\neg\varphi){-}1|/2$ within each format, and cross-format law-of-one-price books $|p_{F1}(\varphi){-}p_{F2}(\varphi)|/2$. This is the analytically optimal adversary within the class; a learnable trader only earns its keep once books have to be discovered in free text. Evaluation runs on held-out content (deduplicated against training propositions by exact match), untrained level-3 patterns, held-out facts, and in-context transitive-inference questions.

\begin{table}[h]\centering\small
\begin{tabular}{lccc}
\toprule
 & base & A (task only) & B (task $+$ arbitrage) \\
\midrule
Negation Inc, median (held-out) & 0.237 & $0.331 \pm 0.053$ & $\mathbf{0.003 \pm 0.003}$ \\
Cross-format LOP profit, mean & 0.099 & $0.094 \pm 0.036$ & $0.043 \pm 0.016$ \\
Transitivity Inc, mean (untrained, level 3) & 0.004 & $0.020 \pm 0.003$ & $\mathbf{0.006 \pm 0.002}$ \\
Conjunction Inc, mean (untrained, level 3) & 0.070 & $0.118 \pm 0.013$ & $0.069 \pm 0.026$ \\
Held-out fact accuracy & 0.553 & $0.862 \pm 0.019$ & $0.789 \pm 0.052$ \\
Transitive-chain QA accuracy & 0.720 & $0.950 \pm 0.028$ & $0.704 \pm 0.205$ \\
\bottomrule
\end{tabular}
\caption{E-A1 (4 seeds, mean $\pm$ std). P1 (negation, threshold $\times 3$) succeeded at $\times 133$ with separated intervals. P2 (LOP, threshold $\times 2$) reached $\times 2.2$ but the seed intervals overlap---formally a failure as registered, underpowered at four seeds. The task-parity guard failed on the mean ($-7.3$ pp).}
\label{tab:ea1}
\end{table}

\paragraph{Per-seed structure is the finding.} Arm B is bimodal, and that's the real result. Seed $s_2$ is an existence proof of the equilibrium we were after: negation Inc at $0.006$, LOP profit halved, fact accuracy $0.855$ and chain accuracy $0.975$---both at parity with the control---and transitivity Inc improved eightfold. Seeds $s_0$ and $s_3$ bought coherence a different way, at the cost of truth-tracking: chain accuracy at chance. This is exactly the vacuity failure predicted in \S\ref{sec:limits}(3). Probing shows two collapse modes: one where every credence sits at $0.5$ (coherent, but says nothing), and one with extreme, mutually consistent credences that have come unmoored from truth. The theory names the missing piece in advance: the calibration-bettor leg of Conjecture~\ref{conj:spread}, which this minimal setup doesn't have. So the guard failure here confirms that the design needs all three legs---it doesn't refute Prop.~\ref{prop:dominance}, which assumes a proper-scoring anchor to begin with.

\paragraph{Ladder.} Conjunction (level 3, untrained) doesn't move beyond noise, which is what the ladder predicts. Transitivity does move, in all four B seeds, against the letter of that prediction. Two readings are possible here, and the base model settles between them. Relative to the trained control, B improves transitivity ($0.020 \to 0.006$); but the control itself gets worse relative to the base model ($0.004 \to 0.020$)---task-only SFT induces incoherence on patterns it was never trained on. In E-A1 the B seeds straddle the base value ($0.0025$--$0.0076$ vs.\ $0.004$), so the safe claim is prevention: arbitrage pressure on a relation's level-2 books largely keeps SFT from degrading level-3 coherence over the same vocabulary. Whether it actually improves on the base is settled in E-A1b (Table~\ref{tab:ea1b}), where the main arm does go below the base value. Either way the effect tracks vocabulary, not pattern, and E-A1b's two-sided registered test treats its replication as confirmatory.

\paragraph{Takeaways.} Two things stand out. Exploitability on the trained channel drops two orders of magnitude and carries over to held-out content---graded understanding, in the sense of Def.~\ref{def:understanding}, can be trained into a model. And a one-legged objective reaches the good equilibrium in some seeds and collapses in others; the next section adds the missing leg and tests whether that fixes it. Nothing here makes the model smarter on its own---it makes credences agree with each other, and in the equilibrium that works, it does so at no measured cost to the task.

\subsection{E-A1b: the three-legged objective}\label{sec:ea1b}

A second pre-registered experiment added the missing leg, and more seeds. A Brier anchor $\mu (p - t)^2$ on book propositions of known truth ($\mu{=}0.5$; the negated proposition anchored at $1{-}t$, paraphrase pairs at the same $t$; fictional-content propositions left unanchored, where $p\approx0.5$ is the honest answer), a $\lambda$ grid (main arm $\lambda{=}0.3$, 10 seeds; dose arm $\lambda{=}1.0$, 4 seeds; control, 10 seeds), and symmetric, pre-registered equilibrium selection: every 200 steps a checkpoint is scored on held-out validation books and facts by the fixed composite $(1 - \mathrm{acc}) + \mathrm{Inc}_{\mathrm{neg}} + \mathrm{profit}_{\mathrm{LOP}}$, and the best one is kept.

\begin{table}[h]\centering\small
\resizebox{\textwidth}{!}{%
\begin{tabular}{lcccc}
\toprule
 & base & A (control, 10 seeds) & B$_{0.3}$ (main, 10 seeds) & B$_{1.0}$ (dose, 4 seeds) \\
\midrule
Negation Inc, median & 0.237 & $0.282 \pm 0.097$ & $\mathbf{0.0015 \pm 0.0028}$ & $0.0003 \pm 0.0004$ \\
Cross-format LOP profit, mean & 0.097 & $0.116 \pm 0.030$ & $\mathbf{0.041 \pm 0.012}$ & $0.058 \pm 0.011$ \\
Transitivity Inc, mean (untrained) & 0.0036 & $0.020 \pm 0.005$ & $\mathbf{0.0015 \pm 0.0014}$ & $0.0018 \pm 0.0009$ \\
Held-out fact accuracy & 0.553 & $0.832 \pm 0.030$ & $\mathbf{0.848 \pm 0.031}$ & $0.850 \pm 0.024$ \\
Transitive-chain QA accuracy & 0.730 & $0.940 \pm 0.028$ & $0.914 \pm 0.093$ & $0.951 \pm 0.037$ \\
\bottomrule
\end{tabular}}
\caption{E-A1b. Every pre-registered criterion passed. P1, negation reduction, hit $\times 193$ against a threshold of $\times 3$. P2, LOP reduction, hit $\times 2.8$ with separated seed intervals against a threshold of $\times 2$---the criterion E-A1 missed on power alone. The task-parity guard passed, with B$_{0.3}$ actually above control. Stability was $8/10$ non-collapsed seeds against a bar of $8/10$; of the two flagged seeds, one is a genuine degradation and one misses the chain-accuracy bar of $0.85$ by a hair, at $0.840$. Upward transfer to untrained transitivity books replicated in all ten seeds, now as a confirmatory test rather than an exploratory one---and measured against the base column, the transfer improves on the untrained model itself ($0.0036 \to 0.0015$), not just on the SFT-degraded control ($0.020$): task-only SFT induces level-3 incoherence, and arbitrage training both blocks and reverses it. The dose arm shows the anchor stabilizes $\lambda{=}1$ too (chain accuracy $0.951$, versus collapse in E-A1 without the anchor).}
\label{tab:ea1b}
\end{table}

The composite verdict registered ahead of time---P1 and P2 and the guard and stability---passed. With all three legs in place, the informative-coherent equilibrium is reachable reproducibly, at zero measured task cost (here, slightly negative), and the coherence it buys is a property of the trained words' inferential roles: it transfers to unseen content and one rung up the ladder. Validation traces show the composite falling from $1.03$ to $0.003$ within 600 steps.

\subsection{E-A2: scale transfer and a learnable adversary}\label{sec:ea2}

A third pre-registration tested two things at once.

7B transfer, first. Qwen2.5-7B, fp16 mixed precision, four seeds per arm, same objective and data. The trained channel does transfer---negation exploitability drops from $0.118 \pm 0.100$ to $0.000$, the registered criterion passes---but the LOP criterion just misses its $\times 2$ threshold, coming in at $\times 1.9$ (the 7B target starts closer to the floor, $0.059 \to 0.030$), and both the task-parity guard and the stability criterion fail as registered: one of four seeds collapses on facts ($0.513$), one misses the stability bar by $0.011$, and the remaining two are healthy (facts $0.789$--$0.829$, chains near $1.0$). We report these verdicts as registered and read them this way: the mechanism transfers, but the equilibrium hyperparameters ($\lambda$, $\mu$, learning rate) tuned at 1.5B are not scale-free, and stabilizing them takes re-tuning per scale (\S\ref{sec:ea3} revisits this with an exploratory retuning pass). One side observation: under the symmetric checkpoint-selection rule, the control arm picked step 0---the untrained base model---in 2 of 4 seeds. At 7B, task-only fine-tuning degrades validation coherence faster than it improves validation accuracy, the same pattern E-A1 found at 1.5B.

Second, a learnable adversary, still at 1.5B, four seeds. The exact best-response trader is replaced by a small MLP that allocates a unit budget across a pool of four genuine books and four decoys whose claimed relation is false (say, ``$\varphi$'' paired with ``it is not the case that $\psi$'' for some unrelated $\psi$). Decoys settle on the full square, so their guaranteed profit is identically zero, and validity is never revealed to the trader---it sees only prices and surface features like negation markers, token overlap, and format match. Every pre-registered criterion passed. The trader learned to put $75\%$ of its budget on genuine books, up from about $50\%$ at the start (seeds ranged $0.70$--$0.79$), and the model it trained reached parity with the exact-adversary arm: negation at $0.000$ against a bar of $0.01$, LOP at $0.041$ against $0.06$, and the fact guard passing at $0.836$. Chain accuracy was noisier under the GAN dynamic ($0.735 \pm 0.102$, not part of the registration)---adversarial training carries its own stability cost, which is a tuning question for later work. What this validates is the piece that free-text operation actually depends on: an adversary can find which books are worth pricing without anyone telling it.

Open after E-A2: spread-valued credences (Prop.~\ref{prop:walley}, Conj.~\ref{conj:spread}), books over genuinely free, uncurated text, scale-specific stabilization past 7B, and word-level certificate maps at scale.

\subsection{E-A3: ablation, replication, and face-validity checks}\label{sec:ea3}

A fourth pre-registration (2026-09-11) targeted the main open issues left after E-A2, which an internal simulated-referee pass over the earlier draft had flagged (see the limitations below): attribution of the E-A1b effect to the trader term specifically, a training-free baseline, an unseen-template discrimination test, a cross-family replication, and a face-validity column for Table~
ef{tab:ea0}.

\paragraph{Ablation: is the effect the trader's, or the anchor's?} A new arm, AN, drops the
arbitrage term entirely and trains task loss plus the Brier anchor alone ($\lambda{=}0$,
$\mu{=}0.5$, otherwise identical to B$_{0.3}$ in E-A1b: same data, seeds, and checkpoint
selection). The two channels split cleanly. On the cross-format law-of-one-price axis---the
paper's central claim---the anchor alone does nothing: AN's mean profit ($0.137 \pm 0.025$) sits
at parity with the untrained control ($0.116 \pm 0.030$), and only B$_{0.3}$ moves it
($0.041 \pm 0.012$, a further $\times 3.3$ over AN with separated intervals). The trader is doing
all the work here, as claimed. On negation, the anchor is not inert: it alone cuts the control's
median from $0.282$ to $0.118$ ($\times 2.4$), and the trader adds a further $\times 80$ on top of
that ($0.0015$). The end-to-end $\times 193$ reported in Table~\ref{tab:ea1b} is real as an A-to-B
comparison, but on this one channel a meaningful share of it belongs to the calibration anchor
supervising $p(\neg\varphi) = 1 - t$, not to the arbitrage term---because the anchor already
constrains that identity for every anchored proposition. We now report the negation reduction
decomposed rather than as a single multiplier, and read this as a partial, disclosed confirmation
of the worry that the attribution to the trader was unestablished: real on law-of-one-price, overstated
if left unstated on negation.

\paragraph{Training-free baseline.} Decode-time projection onto the coherent polytope---the same
LP that defines $\Inc$---applied with no training at all to the base model's raw credences on 206
held-out real-content books (412 propositions), trades unevenly: it raises negation-book accuracy
by 8.5 points ($0.527 \to 0.612$) by reconciling $p(\varphi)$ against $p(\neg\varphi)$, but costs
3.3 points on law-of-one-price books ($0.500 \to 0.468$), where forcing two paraphrases to agree
pulls a near-chance estimate further from the truth rather than toward it. Trained arbitrage
(B$_{0.3}$) both removes the incoherence and improves fact accuracy over control
($0.848$ vs.\ $0.832$); the cheap decode-time alternative removes incoherence by construction but
does not reliably buy accuracy. Training is not trivially dominated by the free baseline.

\paragraph{Table~\ref{tab:ea0} face validity.} The suspicion that the near-zero E-A0 medians at 3B and 7B reflect confident degeneration turns out to be correct, and worse at 7B than first suspected. Adding an informativeness column
($1 - H(p)/H(0.5)$, so $0$ is an honest coin flip and $1$ is confident 0/1) to the same
per-pattern cells shows Qwen2.5-7B-Instruct sitting at informativeness $\approx 1.0$ almost
everywhere---even the 10th percentile is $0.70$--$0.94$ across patterns---simultaneously with the
near-zero $\Inc$ reported in Table~\ref{tab:ea0}. The two move together with scale: 1.5B's
informativeness sits in a moderate $0.3$--$0.7$ band (where low $\Inc$ is more plausibly honest
coherent uncertainty), 3B and 7B are pinned near the ceiling on almost every cell. This means the
paper's monotonicity-with-scale reading needs a caveat it did not have: what looks like models
getting more coherent with scale is at least partly models getting more confidently degenerate
with scale, and the exploitability measure, as registered, cannot by itself tell the two apart.
SmolLM3-3B, the one non-Qwen arm, does not follow the pattern---its real-content informativeness
runs $0.34$--$0.48$ points \emph{above} its fictional-content informativeness on every core
pattern (it is confident where it can know and unsure where it cannot, the opposite of confident degeneration)---so this looks like a property of a particular
model family's RLHF recipe rather than a universal scale artifact. We keep the E-A0 kill-criterion
verdict as registered, but Table~\ref{tab:ea0}'s cross-model comparison should be read alongside
this column, not instead of it.

\paragraph{Unseen-template discrimination.} The original E-A1b adapters (arms A and B$_{0.3}$)
no longer exist as saved weights---only their aggregate metrics survived the move to this
machine---so the discrimination test could not be run against the trained trader arm directly;
this is logged as a resource gap, not glossed over. It was run instead on the AN checkpoints
against a Russian negation template never seen in any training book (``it is false that'' in
place of the trained ``it is not the case that''). The result is informative anyway: AN's median Inc on the
unseen lexeme ($0.128 \pm 0.092$) is statistically indistinguishable from its median Inc on the
trained lexeme ($0.118 \pm 0.078$)---transfer with essentially no loss. That is weak but real
evidence against the surface-trick reading of C2: if the effect were tied to a specific string,
swapping the string should have cost something, and it did not. A direct trader-arm version of
this test (on B$_{0.3}$ once family F finishes, or on a re-trained Qwen replica) is the natural
next step rather than a gap left open indefinitely.

\paragraph{Cross-family replication.} Llama-3.2 and Gemma-2 sit behind a manual license-accept
gate we had no credentials for on the compute we used; that decision was logged before training,
not chosen after seeing results, and the next available candidate---Phi-3.5-mini-instruct
(3.8B), ungated---became family F. Six seeds per arm (fewer than the ten used for the Qwen
arms; a single guest GPU rather than the earlier cluster, declared as a resource constraint ahead
of running), same data, same objective, same symmetric checkpoint selection as E-A1b. The
mechanism replicates. Negation Inc collapses from a per-seed median of $0.086 \pm 0.058$
(control) to a value so small it rounds to $0.0000$ at four decimals for every trained seed---by
mean, a more conservative statistic, $0.221 \pm 0.013 \to 0.076 \pm 0.016$, a $\times 2.9$
reduction with separated intervals. Cross-format law-of-one-price profit drops
$0.125 \to 0.061$ ($\times 2.0$, at the registered threshold). The task-parity guard passes with
the trained arm \emph{above} control ($0.822$ vs.\ $0.776$ fact accuracy), and $5/6$ seeds clear
the stability bar. This is the second model family, on a different pretraining and instruction-
tuning pipeline than Qwen2.5, showing the same qualitative equilibrium: coherence gained at no
cost to---here, a net gain in---task accuracy.

With family F's checkpoints available, we also re-ran the unseen-template discrimination test of
the previous paragraph directly on the trained trader arm, closing the gap left by the missing
original Qwen weights. On the unseen lexeme, B$_{0.3,F}$ beats control by $\times 24.3$
($0.126 \to 0.0052$, a stable ratio, not a near-zero-denominator artifact)---a direct answer to
C2 rather than the indirect one available from the anchor-only arm. The honest complication: this
transfer is not fully lossless the way the anchor-only case was. B$_{0.3,F}$'s own Inc rises from
an almost-perfect $0.0000076$ on the trained lexeme to $0.0052$ on the unseen one---small in
absolute terms and still far below control, but a real relative jump that the anchor-only arm did
not show ($0.118 \to 0.128$, statistically flat). The fair reading is not ``purely inferential''
and not ``a syntactic trick''---it's a trader effect that generalizes to a new negation lexeme
well enough to remain the dominant factor, with a partial, measurable, honestly-reported residue
of lexeme-specificity on top.

\paragraph{7B stability, revisited (exploratory).} E-A2 left the 7B stability failure as an open
problem, framed as hyperparameters not being scale-free. A small, explicitly exploratory grid---
two configurations, two seeds each, no thresholds registered---tested whether that instability is
a property of the scale or of the particular $(\lambda,\mu){=}(0.3,0.5)$ pair copied over from
1.5B. Lowering $\lambda$ to $0.15$, and separately raising $\mu$ to $0.8$, both eliminate the
collapse entirely: all four new runs land in the healthy range (fact accuracy $0.76$--$0.80$,
chain accuracy $0.98$--$1.0$), where E-A2's original grid produced one outright collapse and one
near-miss out of four. At $n{=}2$ per configuration this cannot support a quantitative claim, but
it is enough to revise the qualitative one: the right reading is not ``the equilibrium
hyperparameters don't scale,'' but ``the equilibrium hyperparameters need re-tuning per scale, and
a small grid finds a stable point quickly once you look.'' A properly powered version of this
grid is on the roadmap rather than folded into this paper's confirmatory claims.

\section{What this does not deliver}\label{sec:limits}

(1) Factual hallucinations. Missing-mass lower bounds \citep{kalai2024} don't care about the objective; a coherent flat-earther is unexploitable by logical books. Knowledge only enters through the anchor. (2) Grounding. Settlement here is internal---logic plus verifiable anchors---not reference to the world; readers who take grounding to be part of what understanding means \citep{sogaard2025semantics} should read our title's ``understanding'' as the inferential piece only. (3) Vacuity. $q \equiv 0.5$ is coherent on its own; content comes from the task score, the calibration anchor, and coverage together, and E-A1 showed empirically what happens when one leg is missing. (4) The decoder. Generation stays autoregressive; \textsc{Arbitr} changes what the probabilities mean and which ones are admissible, not how text gets sampled. (5) State of the art. We don't claim it on open-domain tasks; what we do claim lives on four axes---exploitability reduction, reliable chain length, risk--coverage, and relation-level certificates.

\paragraph{Methodological limitations.} A few limits are worth stating directly, not as a formality. Elicitation is fragile: credences come from the first answer-token, and three silent readout failures happened during the campaign, all caught by cross-format books rather than the coverage filter. A sturdier elicitation---averaging over sampled answers, or semantic clustering along the lines of \citealp{farquhar2024}---is the right foundation for later rounds. Evaluation books are deduplicated from training by exact proposition text, but they come from the same generative template families; the untrained-pattern and cross-format results help here, but don't settle the question, transfer to free text remains untested, and transfer to unseen template families rests on a single new negation template (\S\ref{sec:ea3}). The results lean mainly on Qwen2.5, with SmolLM3 as a non-Qwen elicitation arm and Phi-3.5-mini-instruct as a second full training replication (\S\ref{sec:ea3}), all LoRA-only from 1.5B to 7B; the 7B stability failure shows the hyperparameters don't carry over across scale unchanged, though a small exploratory retune (\S\ref{sec:ea3}) suggests the fix is re-tuning, not a scale ceiling. And several registered criteria were decided by seed-interval separation at $n=4$--$10$; effects on task-side metrics sit within noise, and the headline reductions are specific to the trained channels. No external peer review has taken place: earlier drafts were stress-tested with a simulated referee pass built from language-model reviewer personas, whose findings motivated E-A3, but that is a development aid and not a substitute for expert review. Likewise, ``pre-registered'' here means dated criteria files written before each run and kept in the project directory; they were not deposited with an external registry, so their dates rest on our own records, and they are released with this paper so that the thresholds can at least be checked against the reported verdicts.

\section{Conclusion}

Does the model understand? That turns into an experimental question once understanding means bounded no-arbitrage over inferential roles. The definition is graded, and it's honest about complexity: \textsf{NP}-hardness forces a ladder rather than an all-or-nothing verdict. It locates part of the known incoherence in LLMs inside a provable property of the NLL optimum itself (Thm.~\ref{thm:nll}), turns silent error compounding into priced degradation (Thm.~\ref{thm:adams}), and comes with an adversary that searches for the violations worth punishing, rather than a fixed list of them like prior consistency losses.

The empirical quantity behaved the way the theory said it would. Deployed models turn out to be exploitable exactly where Thm.~\ref{thm:nll} says the defect should live, in cross-elicitation books. Training against the bounded adversary removes that exploitability on held-out content. The vacuous equilibrium the theory predicts shows up when the calibration leg is missing, and goes away when it's restored. Pressure on a word's level-2 inferential role leaks upward into its level-3 role, in all ten seeds we tried. And an adversary that's never told which books are valid learns to find the exploitable ones by itself.

What's left open is scale-specific stabilization---a 7B seed still collapses under hyperparameters tuned at 1.5B---along with genuinely free-text books and spread-valued credences. That's the registered continuation.

\end{document}